\documentclass[11pt]{article}
\usepackage{lmodern}
\usepackage[round]{natbib} 
\usepackage{microtype}
\usepackage{float}

\usepackage{caption}
\usepackage{booktabs} 
\usepackage{wrapfig}
\usepackage{graphicx}
\usepackage{subfigure}

\usepackage{xcolor}
\definecolor{bubbles}{rgb}{0.91, 1.0, 1.0}

\usepackage{amsmath}
\usepackage{amssymb}
\usepackage{mathtools}
\usepackage{amsthm}

\theoremstyle{plain}
\newtheorem{theorem}{Theorem}[section]
\newtheorem{proposition}[theorem]{Proposition}

\theoremstyle{definition}

\theoremstyle{remark}
\newtheorem{remark}[theorem]{Remark}
\usepackage{multicol, blindtext}
\RequirePackage[colorlinks,citecolor=blue,urlcolor=blue]{hyperref}
\usepackage[capitalize,noabbrev]{cleveref}

\usepackage[textsize=tiny]{todonotes}
\usepackage[abs]{overpic}
\usepackage{colortbl}
\definecolor{LightCyan}{rgb}{0.88,1,1}
\newcolumntype{b}{>{\columncolor{white}}c}

\def\eqref#1{equation~\ref{#1}}

\def\1{\bm{1}}

\newcommand{\R}{\mathbb{R}}
\newcommand{\davg}{\mathsf{AD}_{d}}
\newcommand{\cavg}{\mathsf{AD}_{c}}
\newcommand{\map}{\mathsf{mAP}}
\newcommand{\heter}{\times \mathcal{R}}
\newcommand{\tavg}{\mathsf{AD}_{\Delta}}

\newcommand\Tstrut{\rule{0pt}{2.2ex}}         
\newcommand\Bstrut{\rule[-1.4ex]{0pt}{0pt}} 

\author{ Francesco Di Giovanni\footnote{ Equal contribution.} \\ {\scriptsize Twitter Inc.} \\ {\footnotesize fdigiovanni@twitter.com ~~}  \and   {Giulia Luise$^\ast$} \\ {\scriptsize University College London} \\ {\footnotesize g.luise.16@ucl.ac.uk ~~} \and  Michael M. Bronstein \\ {\scriptsize Twitter Inc., University of Oxford} \\ {\footnotesize mbronstein@twitter.com ~~}  }
\date{}

\title{\LARGE\bf Heterogeneous manifolds for curvature-aware \\ graph embedding\vspace{1em}} 
\begin{document}
\maketitle

\vskip 0.3in




\begin{abstract}
Graph embeddings, wherein the nodes of the graph are represented by points in a continuous space, are used in a broad range of Graph ML applications. 
The quality of such embeddings crucially depends on whether the geometry of the space matches that of the graph. 
Euclidean spaces are often a poor choice 
for 
many types of real-world graphs, where hierarchical structure and a power-law degree distribution are linked to negative curvature. 
 In this regard, it has recently been shown that hyperbolic spaces 
 and more general manifolds, such as products of constant-curvature spaces and matrix manifolds, are advantageous to approximately match nodes pairwise distances. 
%
%
%
However, all these classes of manifolds are {\em homogeneous}, implying that the curvature distribution is the same at each point, making them   
unsuited to match the local curvature (and related structural properties) 
of the graph. 
%
%
%
%
%
%
In this paper, we study graph embeddings in a broader class of \textit{heterogeneous} rotationally-symmetric manifolds. By adding a single extra radial dimension to any given existing homogeneous model, we can both account for heterogeneous curvature distributions on graphs and pairwise distances. We evaluate our approach on reconstruction tasks on synthetic and real datasets and show its potential in better preservation of high-order structures and heterogeneous random graphs generation.    
\end{abstract}

%


\section{Introduction}
Embedding data into a continuum space is at the heart of representation learning. Data are then manipulated and processed in some downstream task and their performance is usually impacted by the power of the representation. For quite some time continuum space was just a synonym of Euclidean space and the main perspective amounted to assuming that data mapped to high-dimensional vectors likely lived on a smaller but \emph{generally curved} embedded submanifold. This idea, known as \emph{manifold assumption} \cite{bengio2013representation}, has inspired many traditional manifold learning algorithms \cite{roweis2000nonlinear, tenenbaum2000global, belkin2001laplacian}.

Recently, a new trend has emerged of encoding the geometry of the data directly into a richer ambient manifold rather than implicitly reconstructing it based on the manifold assumption. This approach has become particularly popular in the context of graphs embeddings. 
Graphs describing many natural systems of relations and interactions 
often exhibit similar properties such as power-law degree distribution and hierarchical structures, that 
are associated with 
hyperbolic geometry \cite{krioukov2010hyperbolic, sarkar2011low}. It is thus not surprising 
that hyperbolic embeddings 
turned to be beneficial for reconstruction tasks such as link-prediction even in low-dimension \cite{nickel2017poincare, chamberlain2017neural, sala2018representation}. The improved performance is due to the ambient space better matching structural properties of the input graph: this greater flexibility of the hyperbolic spaces is encoded in their \emph{curvature} information, contrarily to the flat Euclidean setting. Such findings have sparked interest in exploring different manifold classes such as products of constant curvature spaces (\cite{gu2018learning}, generalizing \citet{wilson2014spherical}), and matrix manifolds \cite{cruceru2020computationally}, that could better accommodate the structural properties of graphs. 

\begin{wrapfigure}{r}{0.5\linewidth}
\centering
\vspace{-10.5mm}\hspace{-2mm}
\includegraphics[width=1.1\linewidth]{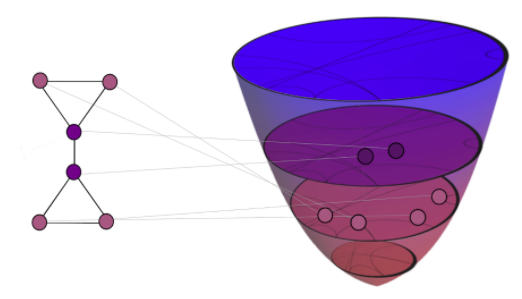}\vspace{-5mm}
\caption{Left: nodes colored according to discrete curvature. Right: a heterogeneous manifold proposed in this paper, colored by its scalar curvature, and the embedding of the graph into it.\vspace{-1mm} 
}
\label{fig:first_page}
\end{wrapfigure}

Only very recently these approaches have been shown to be special instances of graph embeddings into symmetric manifolds \cite{lopez2021symmetric}. This family of manifolds are usually amenable to optimization techniques, 
partly because all points `look the same', a feature known as \emph{homogeneity}. Homogeneity often allows for closed and tractable formulas for distances and exponential maps that are required for Riemannian gradient descent algorithms \cite{wilson2018gradient}, but at the same time makes the ambient space `stiff' since its curvature  is position-independent. 
On the other hand, real-world graphs are often {\em heterogeneous}, where clustering and density generally vary from node to node. This local heterogeneous geometry can be encoded in discrete 
curvature \cite{forman2003discrete, ollivier2007ricci, ollivier2009ricci}, which have recently been used to detect communities \cite{ni2019community} and 
improve information propagation in graph neural networks \cite{topping2021understanding}. 

We propose to go beyond the common strategy of simply minimizing distance-based losses and instead make our embedding \emph{curvature-aware}, by jointly matching both pairwise distances and node-wise curvature information with pointwise curvature on the manifold. This allows us to directly access structural information about the input graph from the \emph{local properties of the manifold} rather than simply from the configuration of the embedded nodes. 

\paragraph{Contributions.} 
In this paper, we propose what to our knowledge is the first method that preserves both distance and curvature information in building graph representations. To this aim, we study a new class of heterogeneous manifolds consisting of a product of a homogeneous factor and a spherically symmetric one. We show that this family generalize \emph{any} homogeneous embedding studied so far and allow to both match the discrete curvature distribution and retain the computational tractability of standard approaches (something \emph{generally lacking} on heterogeneous manifolds). 
We investigate this family of spaces in detail and show that classical optimization techniques, such as Riemannian Stochastic Gradient Descent, extend to our heterogeneous spaces by computing a \emph{single} additional derivative.
We test this new flexible approach on standard reconstruction tasks showing that it minimizes classical metrics of average distance distortion and matches the node-wise discrete curvature of the graph with the pointwise one on the ambient space (see Figure \ref{fig:first_page}). We conduct preliminary experiments that show the potential of our approach for triangles estimation and manifold random graphs.

\paragraph{Related work.} 
Approximate isometric embeddings of graphs and similar objects have been extensively studied in theoretical geometry \cite{gromov1981structures,linial1995geometry,indyk20178} and in computer graphics applications \cite{memoli2005theoretical,bronstein2006generalized}. 
The recent trend of `geometric machine learning' \cite{bronstein2017geometric,bronstein2021geometric} attempts to incorporate geometric inductive biases such as symmetry and equivariance into deep neural networks and more broadly, leverage the geometric structure of the data and learning tasks. 
%
In the context of graphs, much of this research area is propelled by the success of graph neural networks \cite{sperduti1994encoding, goller1996learning, scarselli2008graph, bruna2013spectral}, where the role of non-Euclidean geometry both in the form of representation of the node features \cite{chami2019hyperbolic, liu2019hyperbolic, bachmann2020constant} and tool to investigate limitations of existing architectures \cite{blend,topping2021understanding} is emerging. In general, homogeneous manifolds other than Euclidean space have already been investigated in deep architectures \cite{huang2018building}, for attention mechanisms \cite{gulcehre2018hyperbolic}, general optimization frameworks \cite{bonnabel2013stochastic,becigneul2018riemannian} and  variational autoencoders \cite{skopek2019mixed}, to mention few. Our work continues the line of research of \citet{wilson2014spherical, nickel2017poincare, gu2018learning, cruceru2020computationally, lopez2021symmetric} and generalizes these methods by proposing a new class of ambient spaces and embedding algorithms.  

\section{Graph embedding into manifolds}


\paragraph{The setting.} We consider an undirected graph $G = (V,E)$ with $n$ nodes. For $i\in V$ we let $\mathcal{N}_{i}$ be the \emph{neighbourhood} of $i$ and  $d_i = \lvert \mathcal{N}_{i}\rvert$ the \emph{degree} of $i$. 
The {\em (geodesic) distance}\footnote{Distances are often termed `metrics'. Here, we prefer the term `distance' to avoid confusion with Riemannian metrics. } $d_{G}(i,j)$ is the 
length of the shortest walk connecting nodes $i$ and $j$. 
We are interested in finding a target space $M$ equipped with metric $d_M$ and an embedding $f:V\rightarrow M$ so that the graph can be \emph{reconstructed} from $f$ and $M$. If there exists an \emph{isometric} map $f:V\rightarrow M$, i.e. satisfying $d_{M}(f(i),f(j)) = d_{G}(i,j)$ for all $i,j\in V$, 
then we can perfectly reconstruct the input data. However, given constraints on $M$ such as bounded dimension, a perfect isometric embedding is typically unavailable, and one tries to find a `least-distorting' embedding, in some sense: 
%
the \emph{average (distance) distortion} $\davg(f)$ 
\begin{equation}\label{eq:average_distance_dist} 
\davg(f) := \frac{2}{n(n-1)}\sum_{i,j = 1}^{n}\left\vert 1 - \frac{ d_{M}(f(i),f(j))}{d_{G}(i,j)}\right\vert
\end{equation}
\noindent and the \emph{mean average precision} $\map$ 
\begin{equation}\label{eq:map}
\map(f) := \frac{1}{n}\sum_{i\in V}\frac{1}{d_{i}}\sum_{j\in\mathcal{N}_{i}}\frac{\lvert \mathcal{N}_{i} \cap \mathcal{R}_{i,j}\rvert}{\lvert \mathcal{R}_{i,j}\rvert},
\end{equation}
\noindent are two common criteria, where $\mathcal{R}_{i,j}$ is the set of nodes $z\in V$ such that $d_{M}(f(i),f(z)) \leq d_{M}(f(i),f(j))$. 
We note that while $\davg(f)$ is affected by pairwise distances beyond the 1-hop neighbourhood, the $\map$ is a measure of how well an embedding is able to reproduce the 1-hop neighbourhood of a node disregarding the actual scale of distances.  We refer to Section 4 in \citet{cruceru2020computationally} for a thorough discussion on the matter of choosing the right criterion for embedding distortion.

In general, distortions are inevitable and tend to accumulate on higher-order structures \cite{VERBEEK20161}, which are important in many practical applications such as social networks \cite{benson2016higher} and physical systems \cite{battiston2020networks}. In this case, it is desirable to go beyond pairwise distances and access other types of information on the ambient space to better reconstruct the input data. Discrete graph curvature \cite{forman2003discrete,ollivier2007ricci} is one way of accounting for such structures. In the rest of the paper, we study  embeddings that can both minimize $\davg$ (or maximize $\map$) and account for local graph structural properties by matching the graph curvature distribution with that of a suitable class of target spaces. 



\subsection{Riemannian manifolds} 

A natural class of continuous embedding spaces for  graphs are Riemannian manifolds \cite{petersen2006riemannian}, since they come with a differentiable structure and are hence amenable to optimization methods. Informally, a $d$-dimensional manifold $M$ is a topological space that can be locally identified with Euclidean space via smooth maps: hence for every point $p\in M$ there exists an associated {\em tangent space} $T_{p}M \cong\mathbb{R}^d$.  
Assume we are given a positive-definite inner product $g_{p}:T_{p}M\times T_{p}M\rightarrow \R$. 
If the assignment $p\mapsto g_{p}$ is smoothly compatible with the differentiable structure of $M$, we refer to $g$ as a \emph{Riemannian metric (tensor)} on $M$. 
\paragraph{Geodesics.} The Riemannian metric $g$ induces a distance function $d_{g}$ that measures the length of minimal paths on the manifold $M$ called \emph{geodesics}. An important property of the distance is that $d_{g}^{2}(\cdot,p)$ is smooth locally\footnote{Namely, away from the cut-locus \cite{petersen2006riemannian}.} around $p$, meaning that any loss $\mathcal{L}$ depending on $d^{2}_{g}$ is locally smooth on $M$ and can hence be optimized by first order methods. 
\paragraph{Exponential map.} 
%
Given $v\in T_{p}M$, the {\em exponential map} $\mathrm{exp}_{p}:T_{p}M\rightarrow M$ yields the point in $M$ obtained by travelling for unit time along the geodesic starting at $p$ with initial speed $v$. 
The exponential map plays a key role in optimization on manifolds, allowing 
to update an embedding at $p$ based on gradients of the loss living in the tangent space. 

\paragraph{Embedding.}
The problem of isometric (metric-preserving) embeddings of 
discrete metric spaces (and graph in particular) has been extensively studied both in theoretical and applied literature \cite{linial1995geometry,indyk20178,johnson1984extensions}.
%
%
%
In general, a graph cannot be isometrically embedded into a fixed space; the structure and the dimension of the embedding space have a crucial effect on the embedding distortion. 
Typically, increasing the dimension of the space allows to reduce the distortion, however, it  comes at the expense of memory and computational cost. 
For this reason, one often seeks a lower-dimensional space with `richer' structure that is better suited for the graph. 
%
When using Riemannian manifolds for graph embeddings, 
the `richness' of the space $M$ is manifested in its \emph{curvature}, which we define next. 

\paragraph{Curvature.} For each point $p$ in $M$, and for each pair of linearly independent tangent vectors $u,v\in T_{p}M$, the \emph{sectional curvature} $K_{p}(u,v)$ at $p$ is the Gaussian curvature (product of the minimal and maximal curvatures) of the surface spanned by $\exp_{p}(\{u,v\})$. Given a tangent vector $v$ at $p$, if we `average' the sectional curvatures at $p$ over a set of orthonormal vectors we obtain a bilinear form $\mathrm{Ric}_{p}:T_{p}M\times T_{p}M\rightarrow \R$ called \emph{Ricci curvature}. 
This bilinear map is related to the volume growth rate and the propagation of information \cite{petersen2006riemannian}[Chapter 9]. By computing the trace of $\mathrm{Ric}$, we finally obtain a map $\mathrm{R}:M\rightarrow \R$ called \emph{scalar curvature}. This is the simplest curvature term one can associate with a manifold and the most natural quantity to adopt when fitting the node-wise curvature on a graph.


\subsection{Homogeneous vs heterogeneous manifolds}\label{subsection_homogeneous} For a point $p\in M$, the sectional curvatures $K_{p}:T_{p}M\times T_{p}M\rightarrow\R$ 
encode the local geometry around $p$. When $K_{p}$ is constant (in the sense that there exists $K\in\R$ such that  $K_{p}(u,v) = K$ for any $p\in M$  and $u,v, \in T_{p}M$), then, up to quotients, $M$ is either a sphere ($K > 0$), a Euclidean space ($K = 0)$, or a hyperbolic space ($K < 0$). We refer to this special class as \emph{space-forms}. 
Besides the ubiquitous Euclidean spaces, by far the most common choice for embeddings, space-forms with negative curvature have recently gained popularity for graph representation learning \cite{chamberlain2017neural,nickel2017poincare}. 

More general than space-forms are \emph{homogeneous manifolds}. This class is characterized by the following property: for any points $p,q\in M$ there exists an isometry mapping $p$ to $q$. This means that an observer cannot distinguish the point they are at based on the surrounding geometry. From the isometry-invariance of the curvature it follows that on a homogeneous manifold the sectional curvatures are only functions of the tangent vectors but \emph{not} of the base-point, meaning that $K_{p}(\cdot,\cdot) = K(\cdot,\cdot)$ and hence that the scalar curvature is \emph{constant}. Recently, manifolds other than space forms have been investigated to better accommodate graphs, as products of space-forms \cite{gu2018learning} and matrix manifolds \cite{cruceru2020computationally, lopez2021vector}. All these spaces fall in the homogeneous class (see Appendix \ref{appendix_sec_Riemann}),
implying that 
one {\em cannot} leverage their curvature to encode \emph{any node-wise} graph information. 
 
To overcome this rigidity one has to drop the homogeneity requirement and consider more general manifolds, i.e. with non-constant (scalar) curvature, here simply referred to as \emph{heterogeneous} (see Figure \ref{figure:hom_vs_heter}). In \cref{sec:het_manifolds} we show that a subclass of such manifolds are good candidates for embeddings that also account for discrete curvature on graphs.





\begin{figure*}[t!]
    \centering
 
\begin{subfigure}[Space forms (sphere, plane, hyperboloid)]{%
\label{fig:spaceforms}
\includegraphics[height=3.5cm]{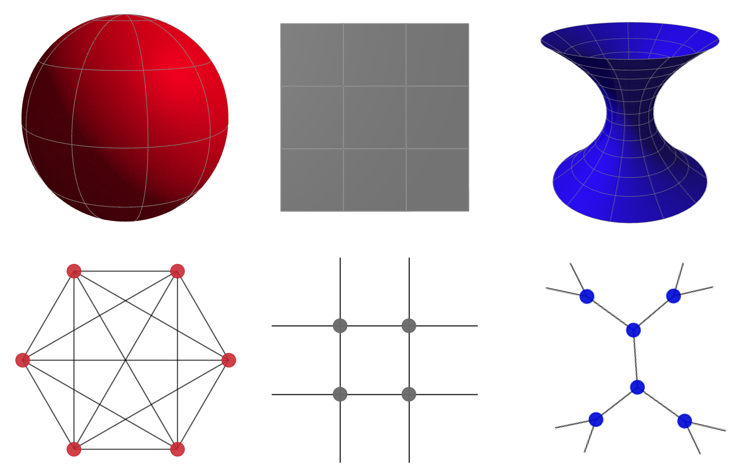}}%
\end{subfigure}
\hspace{10mm}
\begin{subfigure}[Product manifold]{%
\label{fig:productman}
\includegraphics[height=3.5cm]{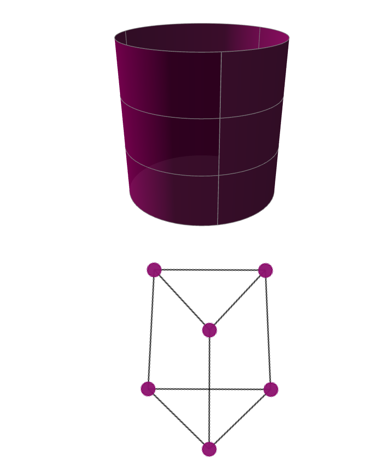}}%
\end{subfigure}
 \hspace{10mm}
\begin{subfigure}[Heterogeneous manifold]{%
\label{fig:hetman}
\hspace{3mm}\includegraphics[height=3.5cm]{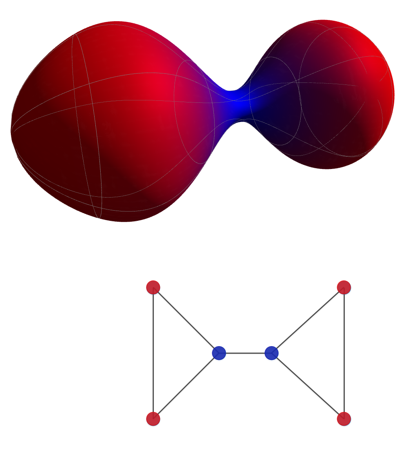}\hspace{3mm}}%
\end{subfigure}
\vspace{-2mm}
 \caption{\emph{Homogeneous} manifolds include (a) space-forms and (b) products of space-forms (e.g. cylinder) and have constant scalar curvature. Below their graph counterparts, with (b) a product of an edge with a triangle, have constant node-wise Forman curvature. A \emph{heterogeneous} manifold (c) has non-constant scalar curvature and its graph counterpart has varying node-wise Forman curvature. } \label{figure:hom_vs_heter}
\end{figure*}

\subsection{Curvature on graphs}

Although a graph does not come with a differential structure, synthetic notions of graph curvature have been introduced, most notably, by \citet{forman2003discrete} and \citet{ollivier2007ricci, ollivier2009ricci}. In both cases, the idea is to consider an edge-based map that can recover some aspects of the Ricci curvature on manifolds, 
%
including relations to the 
volume growth rate 
\cite{paeng2012volume}. Discrete curvatures have been recently used in clustering algorithms \cite{ni2019community} and to detect topological bottlenecks inside a graph that may harm the performance of graph neural networks \cite{topping2021understanding}. While the model we present here works with any notion of  curvature on graphs, we focus on arguably the simplest and most computationally-efficient construction that we define next. 

\paragraph{Augmented Forman curvature.} Following \citet{samal2018comparative} we define the {\em $\gamma$-augmented Forman curvature} of an unweighted\footnote{Generalizations of such formula to weighted graphs exist and can be easily extended to our framework.} graph $G = (V,E)$ as the map $\mathrm{F}:E\rightarrow\R$
\begin{equation}\label{eq:Ricci_Forman}
\mathrm{F}(i,j) = 4 - d_{i} - d_{j} + 3\gamma\,\sharp_{\Delta}(i,j), \quad \gamma > 0,
\end{equation}
\noindent where $d_{i},d_{j}$ are the degrees of $i,j\in V$ respectively and $\sharp_{\Delta}(i,j)$ is the number of triangles based at the edge $(i,j)$. We note that $\gamma$ regulates the contribution of triangles and is generally set equal to one. 

\begin{wrapfigure}{r}{0.5\linewidth}
    \centering
      \includegraphics[scale=0.4]{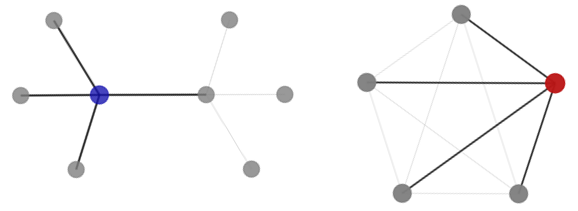}

    \caption{Nodes with same degree but different curvature.  
    }
    \vspace{-4mm}
    \label{fig:degree_vs_curv}
\end{wrapfigure}

We highlight how in spirit the Forman curvature is related to the dispersion of edges (geodesics) similarly to the continuous setting: if edges departing from adjacent nodes $i\sim j$ \emph{converge} in common neighbours (they form triangles), then the graph is \emph{positively curved} along the edge $i\sim j$, in analogy with the spherical case where geodesics starting parallel at the equator converge at the pole. Similarly, trees constitute the discrete analogy of hyperbolic spaces. 


As for manifolds, we can trace over the edges passing through $i$ to compute the node-wise Forman scalar curvature:
\begin{equation}\label{eq:Forman_scal}
\mathrm{F}(i) = \frac{1}{d_{i}}\sum_{j:j\sim i}\mathrm{F}(i,j)
\end{equation}
\noindent The signal $\mathrm{F}:V\rightarrow \R$ encodes information about the 2-hop neighbourhood of a node. Indeed, we can have nodes with same degrees but very different (scalar) Forman curvatures describing distinct geometric configurations as in Figure \ref{fig:degree_vs_curv}.

\paragraph{Distortion of embeddings and the role of curvature.} In principle, if we are able to find an isometric embedding $f$ of $G$ into some space $M$, meaning that the average distance distortion $\davg$ vanishes, then we have preserved all the information and can thus fully reconstruct the graph. We have already noted though how distortions are often unavoidable for a given ambient space \cite{VERBEEK20161}. Partly motivated by these findings, we wish to leverage the structural information encoded in the Forman curvature -- which is particularly relevant for higher order structures as cliques --
to find embeddings that minimize $\davg$ \textit{and} match the node-wise curvature of the graph with the point-wise curvature of the manifold. 
As recalled in \cref{subsection_homogeneous}, homogeneous manifolds, where the curvature is independent of the base point, are only suitable for graphs where all the nodes have the same scalar Forman curvature (e.g. cliques, cycles, or product graphs).
 For more general (heterogenous) graphs, we need to consider embeddings into heterogeneous manifolds where the curvature information changes at each point, which we discuss next. 

\section{Spherically symmetric heterogeneous spaces}\label{sec:het_manifolds}
The family of manifolds that we study are characterized by two features: a product structure and spherical\footnote{We use \emph{rotational} and \emph{spherical} interchangeably.} symmetry. These two elements play a key role in ensuring that curvature and distance can be computed in \emph{closed formulas}, something generally uncommon on manifolds.  
\paragraph{Product manifolds.}
Given two Riemannian manifolds $(M_{1},g_{1})$ and $(M_{2},g_{2})$, their Cartesian product $M := M_{1}\times M_{2}$ can be equipped with a standard Riemannian structure $g := g_{1}\oplus g_{2}$. 
\noindent The product structure allows to easily compute relevant quantities such as distance, exponential map and scalar curvature from each factor. The squared distance function on $(M,g)$ and the exponential map are given by  
\begin{align}\label{distanceformula}
& d_{g}^{2}((p_{1},p_{2}),(q_{1},q_{2})) = d_{g_{1}}^{2}(p_{1},q_{1}) + d_{g_{2}}^{2}(p_{2},q_{2}),\\
\label{exponentialformula}
& \exp_{g}|_{(p_{1},p_{2})}(X) = (\exp_{g_{1}}|_{p_{1}}(X_{1}),\exp_{g_{2}}|_{p_{2}}(X_{2})),
\end{align}
respectively, where $p=(p_1,p_2)$ and $q=(q_1,q_2)$ are points on $M$.
The scalar curvature is simply given by the sum 
\begin{equation}\label{scalarcurvatureformula}
\mathrm{R}_{g}(p_{1},p_{2}) = \mathrm{R}_{g_{1}}(p_{1}) + \mathrm{R}_{g_{2}}(p_{2}).
\end{equation}
\noindent The above decomposition shows that for optimizing a distance and curvature dependent loss on a product manifold it suffices to follow the Riemannian gradient descent {\em on each factor separately}.

\paragraph{Rotationally symmetric manifolds.}
Consider polar coordinates $\{(r,\theta,\psi)\}$ in $\R^{3}$. We can write the Euclidean metric in such coordinates as
$g_{E} = dr^{2} + r^{2}(d\theta^{2} + \sin^{2}(\theta)d\psi^{2}) = dr^{2} + r^{2}g_{\mathbb{S}^{2}}$,
where $g_{\mathbb{S}^{2}}$ is the standard metric on the 2-sphere. We can generalize this construction to the class of metrics that are invariant under rotations and can hence be written in polar coordinates as 
\begin{equation}\label{eq:rot_symmetric_metric}
g_{\varphi} = dr^{2} + \varphi^{2}(r)g_{\mathbb{S}^{2}},
\end{equation}
\noindent for some smooth function $\varphi$. Some contraints on $\varphi$ are necessary to build a valid metric (see \cref{sec_append_rot_sym}). 
It is worth emphasizing that the choice $\varphi(r) = \sinh(r)$ recover the Hyperbolic space as well. The explicit formula for the scalar curvature of a rotationally symmetric metric on $\R^{3}$ is (see \cref{sec_append_rot_sym}):

\begin{equation}\label{eq:scalar_c_rs}
\mathrm{R}(r) = 2\left(\frac{-2\partial^{2}_{rr}\varphi}{\varphi} + \frac{1-(\partial_{r}\varphi)^{2}}{\varphi^{2}}\right)(r).     
\end{equation}

The curvature depends on the radial coordinate $r$, meaning that $r\mapsto \mathrm{R}(r)$ is \emph{non-constant} and hence that the resulting space is \emph{heterogeneous} (except for very particular choices of $\varphi$). We emphasize how the curvature is instead independent of angular coordinates $\theta,\psi$. 

We also note that given two points $p,q$ lying along the same ray, i.e. $p = (r_{0},\bar{\theta},\bar{\psi})$ and $q = (r_{1},\bar{\theta},\bar{\psi})$ for some angles $\bar{\theta},\bar{\psi}$, then we have a simple formula for their distance:
\begin{equation}\label{radialdistance}
d_{g}(p,q) = \lvert r_{1} - r_{0} \rvert.
\end{equation}

In the following, we pick one specific instance of rotationally symmetric space, 
by choosing the radial function $\varphi$ in \cref{eq:rot_symmetric_metric} of the form 
\begin{equation}\label{eq:explicit_varphi}
\varphi_{\alpha}(r) = \alpha \arctan\left(\frac{r}{\alpha}\right), \quad \alpha > 0.
\end{equation}



\begin{wrapfigure}{r}{0.35\linewidth}
    \centering\vspace{-6mm}
    \includegraphics[scale=0.25]{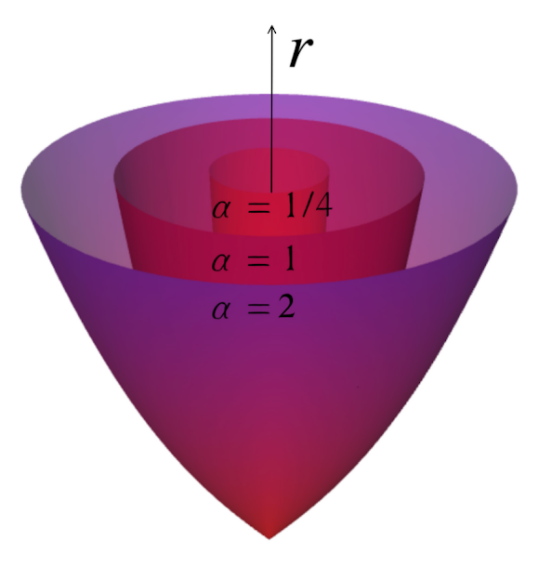}\vspace{-4mm}
    \caption{Visualisation of 3D rotationally symmetric manifolds for different $\alpha$.}
    \vspace{-3mm}
    \label{fig:role_alpha}
\end{wrapfigure}
\noindent While other options are viable, 
this choice is motivated by 
some convenient properties of the metric in \cref{eq:rot_symmetric_metric}. The resulting geometry resembles a hemisphere glued to a cylinder with positive monotone decreasing scalar curvature while $\alpha$ controls the radius of the hemisphere and hence how curved the space is (Figure \ref{fig:role_alpha}). In practice, the value $\alpha$ depends on two tunable hyperparameters that allow us to vary the range of curvatures on the manifold to match the discrete curvature of the embedded graph (see \cref{sec_append_algorithm}). Alternatively, $\alpha$ - or more generally $\varphi$ - could be learned based on the problem. 

\paragraph{Tractable heterogeneous manifolds.} Choose \emph{any} homogeneous manifold $(M_{h},g_{h})$ with a closed formula for the distance and consider a rotationally symmetric space (e.g. with $\varphi$ as in \cref{eq:explicit_varphi}). We build a heterogeneous space as the product $ M_{h}\times\R^{3}$ equipped with the metric $g_{h}\oplus g_{\varphi}$.

\section{Graph embedding in heterogeneous spaces}

We now describe how one can effectively rely on the structure of the heterogeneous manifolds introduced above to learn curvature aware graph representations (see \cref{fig:first_page}). 
%
%
%
\\Given a graph $G$ with node set $V$ and a manifold $M = M_{h}\times\R^{3}$, with $M_{h}$ a homogeneous space, we aim to find an embedding of the form
\[
V \ni x \mapsto y(x) = (z(x),r(x),\bar{\theta}),
\]
\noindent with $z(x)\in M_{h}$ and $(r(x),\bar{\theta})$ polar coordinates in $\R^{3}$ for some fixed angles $\bar{\theta}\in \mathbb{S}^{2}$. According to \cref{distanceformula} and \cref{radialdistance}, the squared distance between two points $y_{0} = (z_{0},r_{0},\bar{\theta})$ and $y_{1} = (z_{1},r_{1},\bar{\theta})$ is
\begin{equation}\label{eq:model_distance}
d_{g}^{2}((z_{0},r_{0},\bar{\theta}),(z_{1},r_{1},\bar{\theta})) = d_{h}^{2}(z_{0},z_{1}) + (r_{1} - r_{0})^{2},
\end{equation}
\noindent where $d_{h}$ is the distance on the homogeneous factor. On the other hand, from \cref{scalarcurvatureformula} we derive that for any $(z,r,\bar{\theta})$ in $M$ the following holds:
\begin{equation}\label{eq:model_curvature}
\mathrm{R}_{g}(z,r,\bar{\theta}) = \mathrm{R}_{g_{h}}(z) + \mathrm{R}_{\alpha}(r) = \mathrm{R}_{h} + \mathrm{R}_{\alpha}(r)
\end{equation}
\noindent where we have used that the scalar curvature of the homogeneous factor $\mathrm{R}_{h}$ is a constant and $\mathrm{R}_{\alpha}(r)$ is the scalar curvature of the rotationally symmetric factor given by \cref{eq:scalar_c_rs} with $\varphi_{\alpha}$ as in \cref{eq:explicit_varphi}. Note that by embedding the nodes along a ray in the rotationally symmetric factor, the angles enter neither the distance function nor the curvature one and are hence geometrically meaningless. Accordingly, we think of our embedding as simply adding a radial coordinate to our chosen homogeneous space $x\mapsto (z(x),r(x))$ to obtain a heterogeneous curvature now varying with $r$. Therefore, we \textbf{simplify our notation} and denote this class of embedding spaces by $M_{h}\heter$ to emphasize that \emph{there is only one additional dimension compared to the homogeneous baseline}. 
%
From now on, we tacitly assume that any heterogeneous graph embedding is of this form. 
Note in particular that if we embed the graph into $r = \text{const}$, we recover existing homogeneous models. In the following, we usually take $(M_{h},g_{h})$ to either be a space-form or a product thereof as in \citet{gu2018learning}.

\paragraph{Example.} Let $\mathbb{H}^{d}$ denote the standard $d$-dimensional hyperbolic space. Since the scalar curvature of $\mathbb{H}^{d}$ is given by $\mathrm{R}_{\mathbb{H}^{d}} = -d(d-1)$, if we consider for example 
an embedding $x\mapsto (z(x),r(x))\in \mathbb{H}^{5}\times \mathbb{H}^{5}\heter$, then \cref{eq:model_curvature} becomes
$\mathrm{R}(z(x),r(x)) = 2\left(-5\cdot 4\right) + \mathrm{R}_{\alpha}(r(x))$.

\paragraph{Loss function.} 
Similarly to \citet{nickel2017poincare, gu2018learning, cruceru2020computationally}, we construct embeddings by minimizing a suitable loss function. 
Thanks to \cref{eq:model_distance} and \cref{eq:model_curvature} we can minimize any distance and curvature depending loss via gradient descent.  
Let $V = \{x_{1},\ldots,x_{n}\}$ and denote the embedded nodes by 
$\{y_{i}=(z_i, r_i)\}_{i=1}^n\subset M_{h}\heter$. 
%
Using the shorthand $(y_{1},\ldots, y_{n}) = \{y_{i}\}$, we consider a loss function 
of the form
\begin{equation}\label{eq:loss_defn}
\mathcal{L}(\{y_{i}\}) = \mathcal{L}_{\text{d}}(\{y_{i}\}) + \tau\mathcal{L}_{\text{c}}(\{y_{i}\}),
\end{equation}
\noindent 
where $\tau$ is a scale parameter acting as a regularizer. We take $\mathcal{L}_{d}$ to be the average relative squared distance distortion (also known as `dilation')
\begin{equation}\label{eq:loss_distance} 
\mathcal{L}_{d}(\{y_{i}\}) = \sum_{i,j}\left \vert \frac{d_{g_{h}}^{2}(z_{i},z_{j}) + (r_{i} - r_{j})^{2}}{d_{G}^{2}(x_{i},x_{j})} - 1 \right\vert    
\end{equation}
\noindent where we have used \cref{eq:model_distance} to compute the squared distance of embedded nodes. This has the advantage of depending only on the squared distance functions which we recall to be locally smooth. On the other hand, $\mathcal{L}_{c}$ is a new curvature-based loss
\begin{equation}
\mathcal{L}_{c}(\{y_{i}\}) = \sum_{i}\frac{\left(\mathrm{F}(x_{i}) - \mathrm{R}_{h} - \mathrm{R}_{\alpha}(r_{i})\right)^{2}}{\left(\lvert \mathrm{F}(x_{i})\rvert + \epsilon\right)^{2}}, \label{eq:loss_curvature}
\end{equation}
\noindent where $\epsilon$ is a constant to avoid numerical instabilities. Note that this is just one option to encourage curvature matching and other losses can be adopted. 

By minimizing $\mathcal{L}_{d}$ we account for long-range interactions in the form of pairwise distances meaning that we prioritize minimization of the average distance distortion $\davg$. However, alternative distance-based losses have also been explored and some of these are more tailored to better recovering the neighbours (as measured by the $\map$) as discussed in Section 4 of \citet{cruceru2020computationally}. On the other hand, $\mathcal{L}_{c}$ is a curvature-based distortion so it measures \emph{how close the local geometry of the manifold around the embedded nodes resembles that of the graph} $G$. 

\subsection{Algorithm} Given our loss in \cref{eq:loss_defn}, we apply Riemannian stochastic gradient descent (R-SGD) \cite{bonnabel2013stochastic} to find an optimal embedding  of a given graph $G$ into our heterogeneous manifold. Assume that we have chosen a homogeneous factor $M_{h}$ and we have mapped the nodes to the points $(z_{i},r_{i})\in M_{h}\heter$. We denote by $\mathsf{U}$ the update on the homogeneous coordinates $\{z_{i}^{(t)}\}$ based on a step of R-SGD of the loss in \cref{eq:loss_defn}. We show the following:
\begin{proposition}\label{lemma:update} If we apply R-SGD to $\mathcal{L} = \mathcal{L}_{d} + \tau\mathcal{L}_{c}$, the update of the radial component simply becomes: 
\begin{equation}\label{updaterule_rot_sym}
(r_{i}^{(t)} + \partial_{r_{i}}\mathcal{L})_{+}  \leftarrow r_{i}^{(t)},
\end{equation}
\noindent with $(\cdot)_{+}$ the positive part. Therefore, the update on the product space is
\begin{equation}\label{eq:update_product_space}
(\mathsf{U}(z_{i}^{(t)}), (r_{i}^{(t)} + \partial_{r_{i}}\mathcal{L})_{+})) \leftarrow (z_{i}^{(t)},r_{i}^{(t)}).
\end{equation}
\end{proposition}
\noindent We provide the proof in the \cref{sec_append_algorithm}. We see that compared to the baseline homogeneous models, the only additional term we need to compute for the gradient of $\mathcal{L}$ is a radial derivative.

\remark{} Our construction generalizes to \emph{weighted} products of homogeneous factors and rotationally symmetric spaces. Specifically, this means considering $M_h \times \lambda^{2} \mathcal{R}$ with $\lambda >0$.  
The scaling $\lambda$ impacts \cref{eq:model_distance} and \cref{eq:model_curvature} and results in weighting less the radial contribution to the distance function effectively allowing us to match the node-wise graph curvature with the pointwise continuous one more easily. Since scaling the curvature generally affects the range of curvatures, we can also allow the curvature matching to be up to a linear transformation. This is discussed further in \cref{subsection_append_linearmaps}. 
\begin{figure}[t!]
\centering
\vspace{-20mm}
\includegraphics[width = 0.96\linewidth]{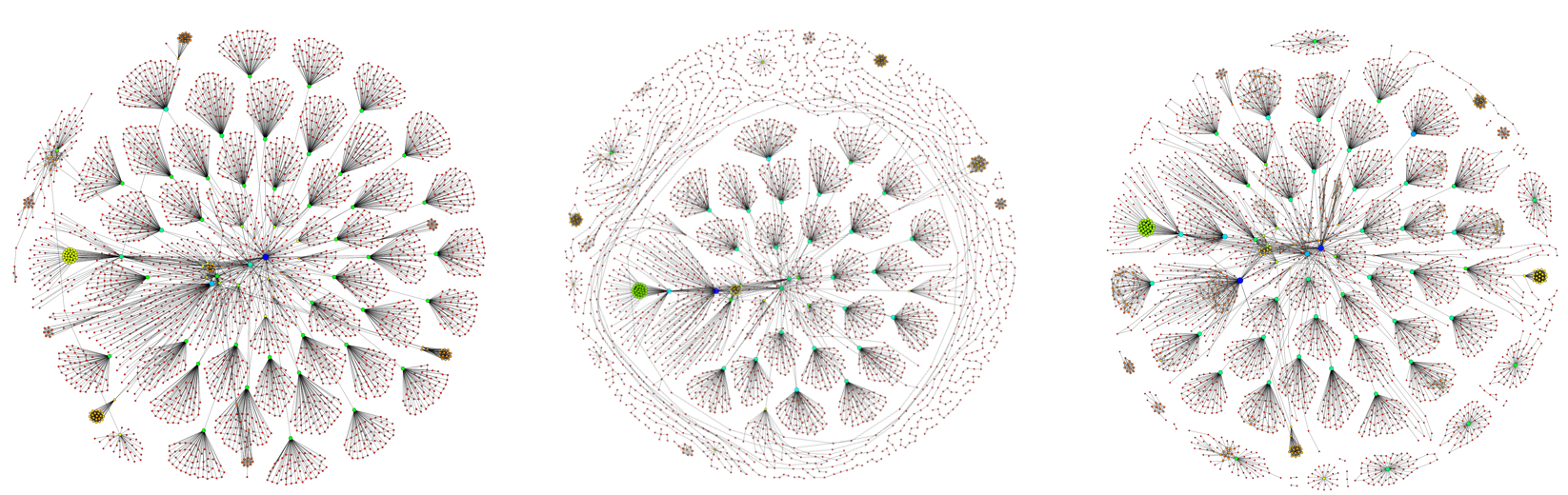}
\caption{Web-edu graph (left) reconstructed using only distance in $\mathbb{H}^{5}\times\mathbb{H}^{5}$ (center) and with a curvature-based correction (with $22.6\%$ reconstruction improvement) in $\mathbb{H}^{5}\times\mathbb{H}^{5}\heter$ (right).}
\label{fig:web_edu_recon}
\end{figure}
 
\section{Experiments}

We now experimentally evaluate graph embeddings into the new class of heterogeneous manifolds introduced in \cref{sec:het_manifolds}. We test the performance on reconstruction of real datasets and showcase its potential for better higher-order preservation and heterogeneous random graph generation. 

\subsection{A synthetic experiment}
Curvature is deeply related to the rate of expansion of space on manifolds, meaning that it affects whether volume of geodesic balls grow polynomially or exponentially \cite{bishop1964geometry}. 
While on homogeneous manifolds the volume of a geodesic ball of given radius is independent of the 
position of its center, on heterogeneous manifolds the volume is \textit{position-dependent}. 
 To highlight this aspect in our setting, we consider a heterogeneous graph composed of a cycle and a tree (\cref{fig:volume_growth}). 
We then embed this graph in $\mathbb{H}^{3}\heter$ and use the normalized volume of the ambient manifold to match the volume on the graph for a given radius $\rho$ (details on the formulas are in \cref{sec_append_rot_sym}). Given a node $x_{i}\in V$ we compare the value of $\lvert \{x_{j} \in V: d_{G}(x_{i},x_{j}) \leq\rho\}\rvert$ with the volume of the annular region of radius $\rho$ centred at the point $f(x_{i})\in \mathbb{H}^{3}\heter$. This shows in a controlled environment how the curvature preservation also allows to recover information related to expansion properties in the graph \emph{directly from the ambient space} \cref{fig:volume_growth} (see \cref{sec_details_exp} for further details).
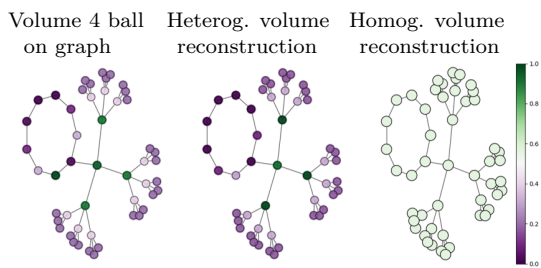
\begin{figure}[b!]
\begin{overpic}[height = 4cm]{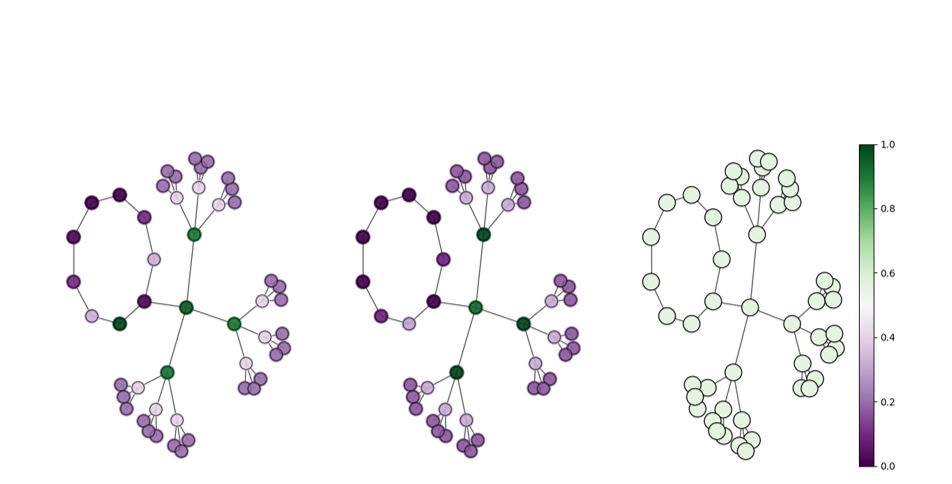}
        \put(10,94){{\scriptsize{Volume 4 ball}}}
        \put(16,84){{\scriptsize{on graph}}}
        \put(70,94){{\scriptsize{Heterog. volume}}}
        \put(74,84){{\scriptsize{reconstruction}}}
        \put(139,94){{\scriptsize{Homog. volume}}}
        \put(143,84){{\scriptsize{reconstruction}}}
      \end{overpic}
    \centering
    \caption{Volume matching.}
    \label{fig:volume_growth}
    \vspace{-5mm}
\end{figure}




\begin{table*}
 \centering
\resizebox{0.8\linewidth}{!}{
\begin{tabular}{ |p{1.75cm}||p{0.8cm}p{0.6cm}p{0.8cm}|p{0.80cm}p{0.6cm}p{0.80cm}|p{0.80cm}p{0.6cm}p{0.80cm}|p{0.80cm}p{0.6cm}p{0.80cm}|  }
 \hline
\multicolumn{13}{|c|}{\textbf{Distance and Curvature Reconstruction Error}} \Tstrut\Bstrut\\
\hline
 \multicolumn{1}{|c||}{ } & \multicolumn{3}{|c|}{\textbf{Aves-Wildbird} } &  \multicolumn{3}{|c|}{\textbf{CS-PhD}} & \multicolumn{3}{|c|}{\textbf{WebEdu}} & 
 \multicolumn{3}{|c|}{\textbf{Facebook}} \Tstrut\Bstrut\\
  \multicolumn{1}{|c||}{ $|V|$ / $|E|$ } & \multicolumn{3}{|c|}{ 131 / 1444  } &  \multicolumn{3}{|c|}{ 1025 / 1043 } & \multicolumn{3}{|c|}{ 3031 / 6474 } & 
 \multicolumn{3}{|c|}{ 4039 / 88324  } \Tstrut\Bstrut\\
  \hline
& $\davg$ & $\map$ & $\cavg$ & $\davg$ & $\map$ & $\cavg$ &$\davg$ & $\map$ & $\cavg$ &$\davg$ & $\map$ & $\cavg$\Tstrut\\[6pt]
$\mathbb{H}^5 \times \mathbb{H}^5 $  &  .088 & .99 &   (131) &  .038    & .96 &  (76.0) & .036    &.98 &  (220) & .043&  .77 & ({\tiny{$>$}}$10^4$)\\[2pt]
\rowcolor{bubbles}
$\mathbb{H}^5 \!\times \mathbb{H}^5 \!\heter\! $   & .085    &.99 &    .001 & .038    &.99 &  .089 & .036    &.97 & .018 & .057    &.76 & .161\\[2pt]

$\mathbb{H}^5 \!\times\mathbb{S}^5 $   & .090    &.98 &   (131) & .050    &.91 &  (76.0) & .050    & .98 & (220) &.059    &.74 & ({\tiny{$>$}}$10^4$)\\[2pt]
\rowcolor{bubbles}
$\mathbb{H}^5 \!\times \mathbb{S}^5 \!\heter\! $   & .088    &.98 &   .004 & .043    &.92 &  .186 & .050    &.99 & .021 & .076    & .74 & .163\\
 \hline
\end{tabular}
}
\caption{\label{table_reconstruction} Perfomance of reconstruction tasks on real datasets. For  homogeneous embeddings, $\cavg$ reports the variance of Forman curvature.} 
\end{table*}

\subsection{Reconstruction tasks} We use our embeddings for the reconstruction of four real-world graphs: 
aves-wildbird \cite{nr} (small animal network)
 CS-PhD \cite{de2018exploratory} (advisor-advisee relationship),
 web-edu \cite{gleich2004fast} (web networks from the .edu domain) and
 Facebook \cite{facebook_dataset} (dense social network). 

We show that our heterogeneous embeddings perform well w.r.t. distance-based metrics (average distance distortion $\davg$ and mean average precision $\map$) while also matching the node-wise curvature information with the pointwise scalar curvature on the manifold. To assess the quality of 
the latter, 
we introduce the \emph{average curvature distortion} 
\[
\cavg := \frac{1}{n}\sum_{x_{i} = 1}^{n}\frac{\lvert \mathrm{F}(x_{i}) - \mathrm{R}(f(x_{i}))\rvert}{\lvert \mathrm{F}(x_{i})\rvert + 1}.
\]
\noindent We stress that the \emph{variance} of Forman is generally high due to its dependence on the size of the degrees \cref{table_reconstruction}. In fact, we have also confirmed experimentally that if we normalize Forman Ricci along each edge using the largest degree of the end-nodes, then $\cavg$ is below $10^{-3}$ on each dataset.
As baseline homogeneous embeddings, we use different 
products of space-forms $M_{h}$ from \citet{gu2018learning} and compare them to the heterogeneous embeddings constructed with the rotationally symmetric factor $M_{h}\heter$. 
%
The results in Table \ref{table_reconstruction} show that the proposed model attains reconstruction fidelity (in the sense of distance distortions) on par with the homogeneous baseline while also minimizing $\cavg$. In the homogeneous setting one can only match an \emph{average} `global curvature' as heuristically investigated in \citet{gu2018learning} since the curvature is position-independent. Accordingly, computing $\cavg$ is meaningless and we then simply report the variance of the Forman curvature as a measure of the information lost when moving from the graph curvature to the smooth manifold one.  

\vspace{3mm}

\paragraph{Estimation of triangle counts.}
Traditional graph embedding tend to distort higher-order structures such as cycles and triangles \cite{VERBEEK20161}. We verify if we can use the curvature in our heterogeneous embeddings to improve triangle counting. 
Given the estimated number of triangles $\sharp_{\Delta}(i)'$ at node $i$, we introduce an \emph{average triangle distortion} $\tavg$ similarly to $\cavg$ by replacing $\mathrm{F}(x_{i})$ with the actual number of triangles $\sharp_{\Delta}(i)$. We consider the graph nodes' embeddings in $M_{h}\heter$ 
and estimate the number of triangles in two different ways:  based on the \emph{nearest-neighbour graph} and  
exploiting the curvature information. In the latter case, we use \cref{eq:Forman_scal} with the curvature of the manifold at respective node embedding in the place of $\mathrm{F}$  
to estimate the number of triangles. 
We report the percentage improvement gained relying on curvature in \cref{tab:improve} and refer to \cref{sec_details_exp} for more details.



Figure \ref{fig:web_edu_recon} shows another example of WebEdu graph reconstruction with and without the use of curvature (see \cref{sec_details_exp} for further details). %
%
It highlights the advantages of heterogeneous embeddings for graph reconstruction tasks by allowing to account for the curvature. 
\vspace{0.3cm}

\begin{table}[t]
\centering
\resizebox{0.55\linewidth}{!}{
\begin{tabular}{ |p{1.55cm}||p{0.68cm}p{0.6cm}p{0.65cm}|p{0.68cm}p{0.6cm}p{0.65cm}|p{0.68cm}p{0.6cm}p{0.65cm}|p{0.68cm}p{0.56cm}p{0.72cm}|  }
 \hline
 \multicolumn{1}{|c||}{ } & \multicolumn{3}{|c|}{\textbf{Aves-Wildbird} } &  \multicolumn{3}{|c|}{\textbf{CS-PhD}} & \multicolumn{3}{|c|}{\textbf{WebEdu}} & 
 \multicolumn{3}{|c|}{\textbf{Facebook}} \Tstrut\Bstrut\\
  \multicolumn{1}{|c||}{  \textbf{Improvement}} & \multicolumn{3}{|c|}{ 49.1 $\%$} &  \multicolumn{3}{|c|}{ 0 $\%$ } & \multicolumn{3}{|c|}{ 44.2$\%$ } & 
 \multicolumn{3}{|c|}{ 20.1$\%$  } \Tstrut\Bstrut\\
  \hline
\end{tabular}
}
\caption{\label{tab:improve} Improvement of triangle counts  using curvature.  Note that the CS-PhD graph only contains 4 triangles.}

\end{table}

\begin{figure}[H]
\centering
 \includegraphics[height = 4cm]{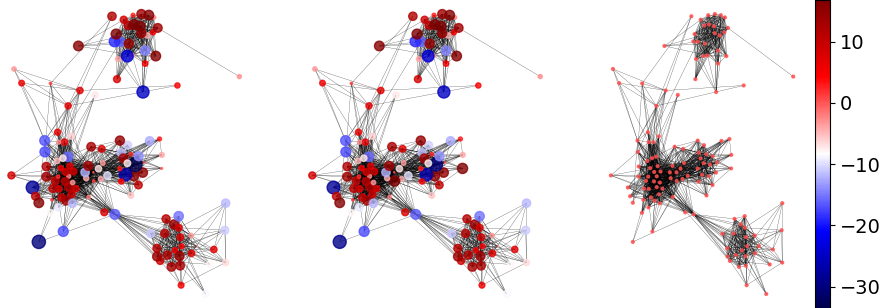}
 \captionof{figure}{Curvature matching on the Aves-Wildbird dataset: true scalar Forman (left), curvature reconstructed from heterogeneous manifold (centre) and from homogeneous manifold (right). 
 }
\end{figure}

\subsection{Generating graphs from heterogeneous manifolds}
Inspired by \citet{cruceru2020computationally}, we study the advantage of heterogeneous manifolds from the perspective of random graphs. The focus is on how one can use the curvature information to generate graphs that are highly \textit{heterogeneous} and exhibit localized dense community structure while preserving scale-free properties common to \emph{small world networks} \cite{watts1998collective, newman2000models, newman2003structure}. 
We consider the following setting: we generate graphs of 500 nodes from $\mathbb{H}^{3}$ and $\mathbb{H}^{3}\heter$ using uniform sampling on the tangent space. We then test two approaches to promote formation of community structures, measured by the size of maximal cliques. On $\mathbb{H}^{3}$, since the only geometric quantity we can vary is the distance, we sample nearest-neighbour graphs with increasing distance threshold. On $\mathbb{H}^{3}\heter$, we instead combine distance and curvature: we increase the distance threshold only for nodes sampled from more positively-curved regions (see \cref{sec_details_exp}). 
\cref{fig:randomgraph_deg_curv} (left-to-right) depicts the degree distributions (averaged over 100 runs using Wasserstein barycenters \cite{agueh2011barycenters}) for three cases: sampling from $\mathbb{H}^{3}$ with unit distance threshold (no dense community structure), with larger distance threshold to encourage clique formation, and relying on curvature in $\mathbb{H}^{3}\heter$ to attain similar clique sizes. We see how the curvature gives rise to heterogeneous density on the graph hence achieving dense community structure while preserving the scale-free profile. On the other hand, sampling from a homogeneous manifold with different thresholds cannot attain dense cliques without losing the power-law degree distribution due to a homogeneous increase in density.

\begin{figure}[H]
\centering
\vspace{-20mm}
\hspace{-6mm}
\begin{overpic}[scale=1]{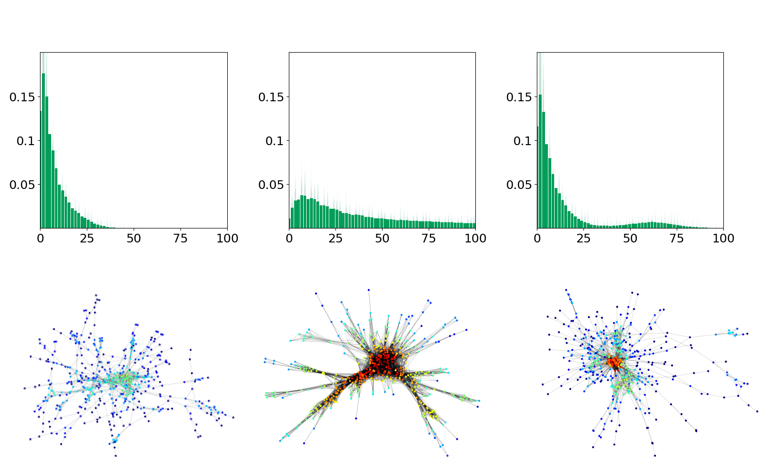}
\put(54,102){{\scriptsize{Degree}}}
\put(178,102){{\scriptsize{Degree}}}
\put(302,102){{\scriptsize{Degree}}}
\put(20,212){{\scriptsize{Homog. w/o community}}}
\put(146,212){{\scriptsize{Homog. w/ community}}}
\put(272,212){{\scriptsize{Heterog. w/ community}}}
\end{overpic}
\caption{Degree distribution and generated examples of graphs sampled from hyperbolic space (left), after the creation of community using distance (centre) and curvature (right). Nodes are colored by the number of triangles (blue corresponding to lower values, red to bigger values).}
\label{fig:randomgraph_deg_curv}
\end{figure}
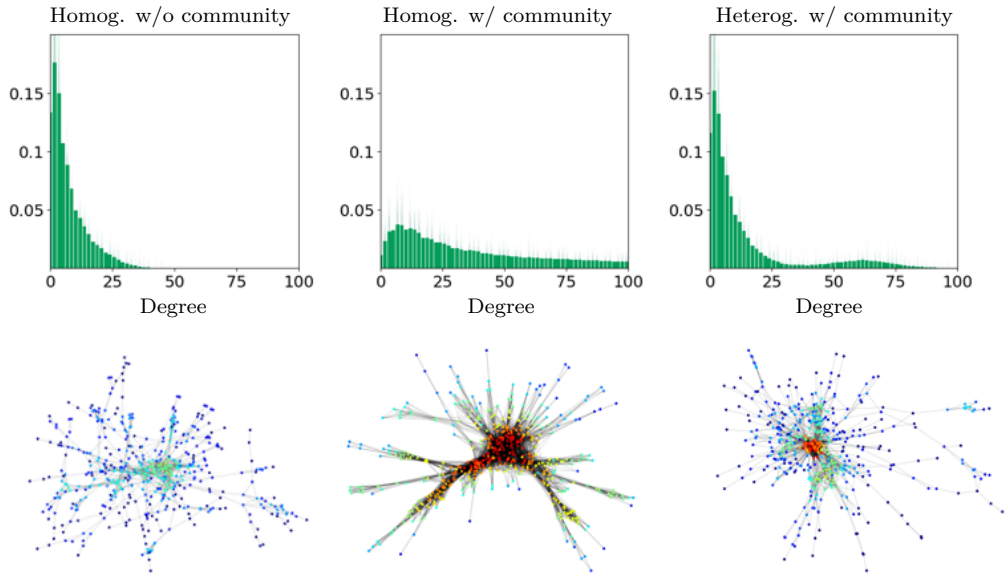

\section{Conclusions}
In this paper we proposed curvature-aware graph embeddings using a novel class of heterogeneous manifolds constructed  as a product of a homogeneous space and a rotationally-symmetric manifold and offering 
a rich heterogeneous geometric structure together with computational tractability. Our approach extends all existing models for graph embeddings to faithfully approximate both distance-based and curvature-based information of the graph. We showed the effectiveness of the proposed approach on real-world graph reconstruction tasks and point towards the advantage of retaining curvature information in higher order structure detection and in controlling local behaviors when sampling random graphs. 

\paragraph{Limitations and future directions.}
We have restricted our discussion to embedding of undirected graphs. Embeddings of directed graphs into  (pseudo)-Riemannian manifolds have recently been studied by \citet{sim2021directed}. In future work, we will study extensions of the proposed framework to such settings. 
Second, we used a fixed rotationally symmetric function $\varphi$ which determines the curvature profile of the ambient space. It is possible to make it \emph{learnable} in future extensions. Third, exploiting Ricci curvature information along walks, more sophisticated anisotropic curvature matching models could be investigated to go beyond the scalar distribution at nodes. Finally, the family of manifolds we consider is a subclass of a more general ensemble of {\em warped products} that will be investigated in the future. In particular, one can also explore using multiple rotationally symmetric copies to better account for tensorial curvature information. 

\paragraph{Acknowledgements} This research was supported in part by ERC Consolidator grant No. 724228 (LEMAN).



\bibliography{main}
\bibliographystyle{plainnat}


\appendix
\onecolumn

\section{Further notions of Riemannian geometry}\label{appendix_sec_Riemann}
The existing graph embedding strategies in symmetric manifolds as per \citet{nickel2017poincare, gu2018learning, cruceru2020computationally, lopez2021symmetric} rely on Riemannian versions of stochastic gradient descent \cite{bonnabel2013stochastic}. Accordingly, we need to compute Riemannian gradients of losses defined on manifolds. In this regard, we first review the notion of Riemannnian gradient.

\paragraph{Gradient on manifolds.} If $(M,g)$ is a Riemannian manifold, a vector field $X$ is a smooth map $X:p\mapsto X_{p}\in T_{p}M$ assigning to each point in $M$ a tangent vector. We recall that tangent vectors on a manifold represent linear differential operators, meaning that for any smooth function $f:M\rightarrow \R$ and vector field $X$, we can construct a smooth function $X(f):p\mapsto X_{p}(f) = df_{p}(X_{p})$.

Given $f:M\rightarrow\R$ smooth, the gradient of $f$ with respect to $g$ is the vector field $\nabla_{g}f$ satisfying $ df(X)|_{p} = g_{p}(X,\nabla_{g}f)$ for any vector field $X$. Given local coordinates $\{x^{i}\}$ around $p\in M$, we can express the gradient of $f$ as 
\begin{equation}\label{eq:gradient_manifold_coordinates}
(\nabla_{g}f)|_{p}^{i} = \sum_{j}(g_{p}^{-1})_{ij}\partial_{j}f(p).    
\end{equation}
\noindent Therefore, if we have a loss $\mathcal{L}$, once we compute its Euclidean gradient with respect to coordinates on $M$ and hence obtain the tangent vector $(\partial_{1}\mathcal{L},\ldots, \partial_{d}\mathcal{L})$, we need to further project it using the inverse metric $g^{-1}$.

Since many real graphs are characterized by features as power-law degree distribution and hierarchical structures that are intrinsic to hyperbolic geometry \cite{krioukov2010hyperbolic}, in most of our evaluations we let the homogeneous factor be a product of space forms containing at least one hyperbolic term. Accordingly, we first briefly review the hyperboloid model we adopt following the discussion in \citet{wilson2018gradient}.

\paragraph{Hyperboloid model.} We consider the Minkowski product on $\R^{d+1}$ defined by 
\[
\langle x, y \rangle_{d:1} := \sum_{i = 1}^{d}x_{i}y_{i} - x_{d+1}y_{d+1},
\]
\noindent with signature $(+,\ldots,+,-)$. The $d$-dimensional hyperbolic space $\mathbb{H}^{d}$ can be described as 
\[
\mathbb{H}^{d} = \left\{ x\in\R^{d+1}: \langle x, x\rangle_{d:1} = -1, x_{d+1} > 0  \right\}.
\]
\noindent The distance between two points on the hyperbolic space is given by
\[
d_{\mathbb{H}^{d}}(x,y) = \mathrm{arcosh}(-\langle x, y\rangle_{d:1}).
\]
\noindent The tangent space $T_{p}\mathbb{H}^{d}$ is isomorphic to the null directions with respect to the Minkowski product at $p$. Finally, the exponential map is given by
\[
\exp_{p}(v) = \cosh(\lvert\lvert v\rvert\rvert)p + \sinh(\lvert\lvert v\rvert\rvert)\frac{v}{\lvert\lvert v\rvert\rvert}.
\]
\noindent Assume now we have a distance-based loss $\mathcal{L}_{d}$ defined on the hyperboloid model. One first has to compute the standard gradient $\nabla\mathcal{L}$ in the ambient space $\R^{d+1}$. Then, in light of \cref{eq:gradient_manifold_coordinates} we rescale using the (inverse) Minkowski metric deriving
\[
\nabla^{d:1}\mathcal{L} = \mathrm{diag}(+,\ldots,+,-)\nabla\mathcal{L}.
\]
\noindent Finally we project $\nabla^{d:1}$ to the tangent space of the hyperboloid to derive $\nabla^{\mathbb{H}^{d}}\mathcal{L}$. 

\paragraph{Homogeneity of existing models.} In this paragraph we briefly review that the main existing models analysed recently in \cite{nickel2017poincare, gu2018learning, cruceru2020computationally, lopez2021symmetric} are indeed homogeneous and hence their curvature information cannot generally match the graph discrete one. In fact, as observed in \cite{lopez2021symmetric}, the homogeneity follows from the stronger requirement of \emph{symmetry} \cite{petersen2006riemannian}[Chapter 8.1]: symmetric spaces are characterized by the property that for each point $p$ there exists an isometry fixing $p$ whose Jacobian is minus the identity. 

In the case of graph embeddings into space forms \cite{wilson2014spherical, nickel2017poincare}, the homogeneity is a known fact and follows for example from the curvature tensor being covariantly constant. This extends to the cartesian products of such spaces analysed in \citet{gu2018learning}. In \citet{cruceru2020computationally} two Riemannian manifolds have been investigated regarding the graph embedding problem: the SPD manifold (i.e. symmetric positive definite matrices) and the Grassmannian manifold. The homogeneity is a consequence of their Lie group structure (see \citet{lin2019riemannian} and \citet{petersen2006riemannian}[Chapter 8.2] respectively). Finally another class of symmetric spaces - Siegel manifolds - have been studied in \citet{lopez2021symmetric}. In particular, we refer to the Appendix of \citet{lopez2021symmetric} for a more detailed discussion of symmetric manifolds and why they are generally advantageous in representation learning tasks \cite{lopez2021vector}.

As observed in Section 2, since the curvature of a Riemannian metric $g$ is invariant under isometries, when $g$ is homogeneous - meaning that its isometry group acts \emph{transitively} - the curvature of $g$ cannot distinguish between two given points $p,q$ otherwise it would break the invariance with respect to the isometry mapping $p$ to $q$. This key property of homogeneous spaces is arguably the main reason why they appear so frequently whenever optimization on manifolds is required: the position-independence of the curvatures makes the geometry of the space the same around each point, often leading to closed and tractable formulas for the distance function and exponential maps, which are generally unavailable. \emph{On the other hand, this rigidity comes with a price: no information about the input graph domain can be derived from the underlying continuum space and its geometry without reconstructing an adjacency skeleton on the embedded point cloud.} 

\textbf{This last point is at the heart of our work, where we are interested in encoding the geometry of the data in the actual continuum texture of the ambient space via curvature matching. To this aim, one has to consider heterogeneous manifolds with position aware curvature.}






\section{Further details on rotationally symmetric spaces}\label{sec_append_rot_sym} 
In this section we outline a few properties of the spherically symmetric manifolds adopted in our construction. We first note that a generally rotationally symmetric metric $g_{\varphi}$ on $(0,\infty)\times \mathbb{S}^{m}$ can be written as 
\[
g_{\varphi} = dr^{2} + \varphi^{2}(r)g_{\mathbb{S}^{m}}
\]
\noindent up to renaming $r$ to be the coordinate representing the distance from the singular orbit (origin in our case), as for example in \citet{angenent2004example, di2021rotationally}. The metric $g_{\varphi}$ on the open manifold $(0, +\infty)\times \mathbb{S}^{m}$ defines a smooth complete metric on $\R^{m+1}$ if $\varphi$ extends to an odd function at the origin with $\partial_{r}\varphi(0) = 1$ \cite{petersen2006riemannian}[Chapter 1.4]. 
\\The curvature information of a spherically symmetric metric on $\R^{m+1}$ is encoded in the sectional curvature $K$ of the 2-plane perpendicular to the spherical orbits and $L$ of the 2-plane tangential to such orbits \cite{petersen2006riemannian}[Chapter 3.2]:
\begin{equation}\label{eq:appendix_sect}
K = -\frac{\partial^{2}_{rr}\varphi}{\varphi}, \,\,\,\,\,\, L = \frac{1 - (\partial_{r}\varphi)^{2}}{\varphi^{2}}.
\end{equation}
\noindent By tracing we derive the formulas for the Ricci curvature and the scalar curvature (coinciding with \cref{eq:scalar_c_rs} in the case $m = 2$ used in the rest of the paper):
\begin{align}
    \mathrm{Ric} &= -m\frac{\partial^{2}_{rr}\varphi}{\varphi}dr^{2} + (-\varphi\,\partial_{rr}\varphi + (m-1)(1 - (\partial_{r}\varphi)^{2}))g_{\mathbb{S}^{m}}, \\
    \mathrm{R} &= m\left(-2 \frac{\partial^{2}_{rr}\varphi}{\varphi} + (m-1)\frac{1-(\partial_{r}\varphi)^{2}}{\varphi^{2}}\right)
\end{align}

\paragraph{Details on our explicit choice.} It follows from \cref{eq:appendix_sect} that any smooth concave function $\varphi$ with $\partial_{r}\varphi(0) = 1$ gives rise to a spherically symmetric metric with nonnegative sectional curvature $K \geq 0,$ $L\geq 0$ and hence $\mathrm{R} \geq 0$. The concavity and monotonicity of the warping function $\varphi$ characterizes strongly the geometry of these spaces and indeed has an impact for example on the type of singularities that the Ricci flow may develop \cite{angenent2004example, di2021rotationally}. 

In particular our choice in \cref{eq:explicit_varphi} reported below
\[
\varphi_{\alpha}: r \mapsto  \varphi_{\alpha}(r) = \alpha\arctan\left(\frac{r}{\alpha}\right), \quad \alpha > 0
\]
\noindent satisfies these properties.

Since $\varphi$ is odd with $\partial_{r}\varphi_{\alpha}(0) = 1$, the manifold $(\R^{m+1},g_{\varphi_{\alpha}})$ is smooth and complete. Moreover, as mentioned above $\mathrm{R}_{\alpha} \geq 0$. In fact, a standard de l'H\^{o}pital argument gives
\[
\mathrm{R}_{\alpha}(0) = \frac{12}{\alpha^{2}}.
\]
\noindent By direct computation one can also check that $\partial_{r}\mathrm{R}_{\alpha} \leq 0$ with 
\[
\inf_{r\geq 0} \mathrm{R}_{\alpha}(r) = \frac{8}{\alpha^{2}\pi^{2}}.
\]
\noindent Note how this is not surprising, since geometrically the manifold looks like a round cylinder away from the origin. In particular, we see that $\alpha$ affects the range of curvature and how positively curved the manifold is at the origin as illustrated in Figure \ref{fig:role_alpha}. We will see below how to choose $\alpha$ based on the range of curvatures on the given input graph we want to match.

\subsection{The role of rescaling.}\label{subsec_append_role_rescaling} We emphasize how the curvature \emph{monotonicity} is in general a feature helping the fitting of the Forman distribution on the graph since otherwise the gradient of the curvature-based loss $\mathcal{L}_{c}$ could get stuck at one stationary point of $\mathrm{R}_{\alpha}$. In fact, this property also allows us to have a better control and interpretability of the hyperparameters entering the model as discussed in the next section. On the other hand, since the graph is still a non-differentiable structure, it may happen that adjacent nodes have highly varying Forman curvature: equivalently, the node-wise Forman curvature may have large Dirichlet energy 
\[
E(\mathrm{F}) = \frac{1}{2}\sum_{i\sim j}\left\vert\left\vert \frac{\mathrm{F}_{i}}{\sqrt{d_{i}}} - \frac{\mathrm{F}_{j}}{\sqrt{d_{j}}} \right\vert\right\vert^{2}.
\]
\noindent Accordingly, to avoid sacrificing the distance-based loss, we grant the model an extra degree of freedom given by a rescaling $\lambda^{2}$ of the rotationally symmetric metric $g_{\varphi}$. We usually take $\lambda$ to be a contraction, meaning that the projection of the distance function on the radial directions in \cref{radialdistance} becomes $\lambda d_{g{\varphi}}((r_{0},\cdot),(r_{1},\cdot))$ hence allowing us to weight less the rotationally symmetric space in the distance-based loss $\mathcal{L}_{d}$. On the other hand the scalar curvature of the rescaled metric transforms as $\mathrm{R}_{\alpha}/\lambda^{2}$. In \cref{sec_append_algorithm} we will describe how to deal with the curvature rescaling and in general make the curvature matching component of our approach more robust to both the choice of the homogeneous factor and of the rescaling factor.

\subsection{Volume growth measurement} Here we describe how we accounted for the volume on the synthetic reconstruction task in Figure \ref{fig:volume_growth}. 
We choose our homogeneous space to be the standard 3-dimensional Hyperbolic space $(M_{h},g_{h}) = (\mathbb{H}^{3},g_{\mathbb{H}^{3}})$.

Since our embedding is spherically symmetric, instead of considering geodesic balls $B_{g}(y_{i},\rho)$ in our heterogeneous space $\mathbb{H}^{3}\heter$, we look at annular regions
\[
V(y_{i},\rho) := \{y = (z,r,\theta):\,\, d_{\mathbb{H}^{3}}^{2}(z,z_{i}) + (r-r_{i})^{2} < \rho^{2}\},
\]
\noindent so that we can explicitly compute volumes and exploit the fact that our geometry (e.g. ordering of nodes with larger volumes at a given radius) is independent of the angular coordinates. We then have
\[
v(y_{i},\rho) := \text{vol}_{g}(V(y_{i},\rho)) = \int_{(r_{i} - \rho)_{+}}^{r_{i} + \rho}\lvert B_{g_{\mathbb{H}^{3}}}(o,\sqrt{\rho^{2} - (r-r_{i})^{2}})\rvert_{g_{\mathbb{H}^{3}}}\,\omega_{2}\,\varphi_{\alpha}^{2}(r)dr,
\]
\noindent with $\omega_{2}$ the area of the 2-sphere and $o$ some fixed point of our hyperbolic model. Say we consider the rotationally symmetric model for the three-dimensional hyperbolic space, then
\begin{equation}\label{eq:volume_growth_formula}
v(y_{i},\rho) = (\omega_{2})^{2}\int_{(r_{i} - \rho)_{+}}^{r_{i} + \rho}\left(\int_{0}^{\sqrt{\rho^{2} - (r-r_{i})^{2}}}\sinh^{2}(z)dz\right)\varphi_{\alpha}^{2}(r)dr.
\end{equation}
\noindent This is the quantity we use to match the volume reconstruction of the graph as in Figure \ref{fig:volume_growth}. 




\section{Details of the algorithm}\label{sec_append_algorithm} We recall the setting we are interested in: assume we want to embed a graph into a heterogeneous space of the form $M_{h}\heter$, for some homogeneous space $M_{h}$. We first prove \cref{lemma:update} stating that to apply R-SGD to our curvature-aware loss $\mathcal{L} = \mathcal{L}_{d} + \tau\mathcal{L}_{c}$ defined in \cref{eq:loss_defn} we only need to account for one additional radial derivative.

\begin{proof}[Proof of Proposition \ref{lemma:update}]
Let $(M_{h},g_{h})$ be a chosen homogeneous manifold. Suppose we have $f:M\rightarrow \R$ smooth with $M = M_{h}\heter$, meaning that we consider the heterogeneous manifold $(M,g) = (M_{h},g_{h})\times (\R^{3},g_{\varphi})$, with $g_{\varphi}$ as in \cref{eq:rot_symmetric_metric} for some smooth radial map $\varphi$. Assume that $f$ is $SO(3)$-invariant, i.e. that given $z\in M_{h}$ and $r > 0$ we have
\[
f(z,r,\theta_{0}) = f(z,r,\theta_{1}),\,\,\,\,\, \forall \theta_{0},\theta_{1}\in \mathbb{S}^{2}.
\]
\noindent We note that this is the case for our loss $\mathcal{L}$ in \cref{eq:loss_defn} which is independent of angular coordinates. Since $g$ is a product metric the tangent space of $M$ is the direct sum of the tangent spaces of the individual factors and we can write the gradient of $f$ as
\[
\nabla_{g}f(z_{0},r_{0},\theta_{0}) = \left(\nabla_{g_{h}}f(\cdot,r_{0},\theta_{0})\right)(z_{0}) \oplus \left(\nabla_{g_{\varphi}}f(z_{0},\cdot,\cdot)\right)(r_{0},\theta_{0}).
\]
\noindent Since $f$ is spherically symmetric, the right hand side becomes
\[
\nabla_{g}f(z_{0},r_{0},\theta_{0}) = \left(\nabla_{g_{h}}f(\cdot,r_{0},\theta_{0})\right)(z_{0}) \oplus \partial_{r}f(z_{0},\cdot,\theta_{0})(r_{0})\partial_{r},
\]
\noindent where we have also used \cref{eq:gradient_manifold_coordinates} and that the inverse metric $g_{\varphi}^{-1}$ writes as 
\[
g_{\varphi}^{-1} = \begin{pmatrix}
1 & 0 \\ 0 & g_{\mathbb{S}^{2}}^{-1}
\end{pmatrix}.
\]
\noindent 

On the other hand, if $\nu\in \R$, the unique $g_{\varphi}$-radial geodesic starting at some $(r_{0},\theta_{0})$ with initial tangent vector $\nu\partial_{r}$ is 
\[
t \mapsto (r_{0} + t\nu, \theta_{0}),
\]
\noindent meaning that the exponential map is always defined at $\nu\partial_{r}$ and is given explicitly by
\[
(\exp_{g_{\varphi}})|_{(r_{0},\theta_{0})}(\nu\partial_{r}) = ((r_{0} + \nu)_{+},\theta_{0}).
\]
\noindent Therefore, we may apply \cref{exponentialformula} and conclude that
\[
(\exp_{g})_{(z_{0},r_{0},\theta_{0})}(\nabla_{g}f) = \left((\exp_{g_{h}})|_{z_{0}}(\nabla_{g_{h}}f(\cdot,r_{0},\theta_{0})(z_{0})),(r_{0}+(\partial_{r}f(z_{0},\cdot,\theta_{0}))(r_{0}))_{+},\theta_{0}\right).
\]
\noindent If we apply the previous computation to each component of $\mathcal{L}$, we then get the update rule for the R-SGD algorithm.
\end{proof}

\subsection{Matching curvature up to invertible linear maps}\label{subsection_append_linearmaps} Next, we discuss how we allowed our algorithm to have an extra degree of freedom to more easily account for distances and curvatures simultaneously. As mentioned in \cref{subsec_append_role_rescaling}, in general the curvature of the heterogeneous model $M_{h}\times\lambda^{2}\mathcal{R}$ writes as
\[
\mathrm{R} = \mathrm{R}_{h} + \frac{1}{\lambda^{2}}\mathrm{R}_{\alpha}
\]
\noindent where we are using the explicit choice in \cref{eq:explicit_varphi} for the spherically symmetric factor with $\lambda > 0$ a positive rescaling on the rotationally symmetric space introduced above. In general, we wish to make our model robust with respect to the choice of the homogeneous factor given that it only leads to a constant value $\mathrm{R}_{h}$ translating the global curvature of our ambient space. Similarly, the role of the rescaling should not affect how we match the node-wise curvature distribution. Accordingly, we propose to reconstruct the curvature information at node $x_{i}$ by the curvature on the manifold at the embedded point $y_{i} = (z_{i},r_{i})$ up to a known shifting and rescaling. It means that for our curvature-based loss we instead minimize
\begin{equation}\label{eq:append_translation_curv_loss}
\mathcal{L}_{c}(y_{1},\ldots,y_{n}) = \sum_{i}\left( \mathrm{F}(x_{i}) - \lambda^{2}\left(\mathrm{R}_{h} + \frac{\mathrm{R}_{\alpha}(r_{i})}{\lambda^{2}}\right) + \rho\right)^{2},
\end{equation}
\noindent where we take the translation $\rho$ of the form
\[
\rho = \lambda^{2}\mathrm{R}_{h} - \min{\mathrm{F}} + \hat{\delta}_{\alpha},
\]
\noindent with $\min\mathrm{F}$ the minimum of scalar Forman on the given graph $G$ and $\hat{\delta}_{\alpha}$ a constant we discuss in the next section. Therefore, \cref{eq:append_translation_curv_loss} becomes
\begin{equation}\label{eq:appendix_curv_loss_no_hom}
\mathcal{L}_{c}(y_{1},\ldots,y_{n}) = \sum_{i}\left( \mathrm{F}(x_{i}) - \mathrm{R}_{\alpha}(r_{i}) -\min\mathrm{F} + \hat{\delta}_{\alpha}\right)^{2}
\end{equation}
\noindent making it \textbf{independent} of the choice of the homogeneous factor $M_{h}$. It remains to discuss the role of $\hat{\delta}_{\alpha}$ and in general how by using the geometry of the problem we can tune two hyperparameters to ensure the curvature matching.

\subsection{Tuning geometric hyperparameters}\label{appendix_hyperparameters} To allow the manifold scalar curvature to fit the node-wise Forman signal, we see from \cref{eq:appendix_curv_loss_no_hom} that a necessary requirement is considering range of curvatures that cover the interval $\max \mathrm{F} - \min\mathrm{F}$. Since by choice $\mathrm{R}_{\alpha}$ is \emph{monotone}, we immediately see that the $\mathrm{argmin}\mathrm{F}$ on the graph should be mapped to the radial coordinate $r_{\min}$ where $\mathrm{R}_{\alpha}(r_{\min}) = \hat{\delta}_{\alpha}$. We know that given $\alpha$, the function $\mathrm{R}_{\alpha}$ admits a horizontal asymptote given by $8/(\pi\alpha)^{2}$, therefore we find the constraint
\begin{equation}\label{eq:append_constraint}
\hat{\delta}_{\alpha} \geq \frac{8}{\pi^{2}\alpha^{2}}.
\end{equation}
\noindent On the other hand, a symmetric argument works for $\max\mathrm{F}$: in fact, if we denote by $r_{\max}$ the radial coordinate we should map $\mathrm{argmax}\,\mathrm{F}$ to, by monotonicity we have the constraint
\[
\frac{12}{\alpha^{2}} = \mathrm{R}_{\alpha}(0) = \mathrm{R}(r_{\max}) + \ell_{+}= \max\mathrm{F} - \min\mathrm{F} + \hat{\delta}_{\alpha} + \ell_{+},
\]
\noindent with $\ell_{+}$ our first hyperparameter \emph{controlling how close to the origin of our spherically symmetric factor we need to be to match the maximum of Forman on the input graph}. In particular, given $\ell_{+}$, the constraint in \cref{eq:append_constraint} yields:
\[
\hat{\delta}_{\alpha} = \frac{2}{3\pi^{2} - 2}\left(\max\mathrm{F} - \min\mathrm{F} + \ell_{+}\right) + \delta
\]
\noindent where $\delta > 0$ is our second hyper-parameter determining what is the range of radial coordinates needed for the curvature matching, since the smaller $\delta$ the closer to its asymptote $\mathrm{R}_{\alpha}$ must be to take on the value $\min\mathrm{F}$. In conclusion, the choice of $\alpha$ is affected by two geometric hyperparameters $\ell_{+}, \delta$ and is a function of the range of Forman curvatures on the given graph we want to embed in our heterogeneous model:
\[
\alpha = \left(\frac{12}{\max\mathrm{F} - \min\mathrm{F} + \delta + \ell_{+}}\right)^{\frac{1}{2}}.
\]

\section{Additional details on Experiments}\label{sec_details_exp}
In this section we expand on our evaluation section further commenting on methods adopted and reporting additional plots.

\subsection{Synthetic experiment}
We consider an embedding of the graph in Figure \ref{fig:volume_growth} into $\mathbb{H}^{3}\heter$. Note that the choice of the graph is emblematic of heterogeneous pattern since nodes inside the cycle would have constant volume growth while nodes in the tree region will have exponential volume growth. 

We fix a radius $\rho = 4$ and we compute the volume of each `geodesic' ball inside a graph, i.e. $\lvert\{x_{j}\in V: d_{G}(x_{i},x_{j})\leq \rho\}\rvert$. We then use the spherical symmetry of our ambient space and the formula in \cref{eq:volume_growth_formula} with $\rho = 4$ and $y_{i}$ given by our embedded nodes. Once we normalize the volume scores on both the graph and the ambient space, we can finally compare the results as in Figure \ref{fig:volume_growth}. We emphasize again that on a homogeneous manifold this information cannot be accessed from the actual continuum space since $\lvert B_{g_{h}}(p,r) \rvert$ is only a function of the radius but not of the base-point if $(M_{h},g_{h})$ is homogeneous.

\subsection{Reconstruction tasks} We summarize here additional details concerning the methods and results of \cref{table_reconstruction}. We use the model proposed in \citet{gu2018learning} to possibly learn optimal constant rescaling of the homogeneous factors $\mathbb{H}^{5}\times\mathbb{H}^{5}$ and $\mathbb{H}^{5}\times \mathbb{S}^{5}$ and we consider a training of 3000 epochs for each dataset and ambient space. Typical values of the scale parameter $\tau$ in the loss are $0.1, 0.01$, noting that this allows to minimize the curvature distortion without penalizing the distance-based one. 
In terms of hyperparameters $\ell_{+}, \delta$ introduced in \cref{appendix_hyperparameters} we usually take large values (especially for the dense network Facebook) in the range $10, 100,1000$ which allow to avoid plateau regions of the scalar curvature profile and hence make the learning easier. This is also accounted for the initialization of the radial coordinate since once again we want to avoid flat regions of the curvature profile: since we have an explicit formula for the curvature this can be done efficiently (usually the initialization is for $r \in (0.1, 1)$). 

\paragraph{Triangle distortion.} Given the point clouds found by the embedding into the manifold, we reconstruct the adjancency matrix as follows: we draw an edge between a pair of nodes $(i,j)$  if the distance between the corresponding embedded points  $y_i$, $y_j$ is lower than a certain threshold $\rho$, i.e. if  $d_M(y_i, y_j) \leq \rho$. Self-loops are then removed. The threshold $\rho$ is tuned on a validation set that is built drawing randomly $10\%$ of the nodes of the dataset. The tuning aims at minimizing the reconstruction error between the reconstructed and real graph: given $A$ the adjacency matrix of the graph and $A_\rho$ the adjacency matrix associated with the $\rho$-nearest neighbour graph, we tune $\rho$ to minimize $\Vert A - A_\rho\Vert$ on the validation set. More advanced ways for graph reconstruction and link predictions exist in the literature (see for example \cite{nickel2017poincare}). 

Given our best nearest neighbour reconstruction adjacency $A_{\rho}$ and our manifold curvature values $\mathrm{R}(y_{i})$ we reconstruct the number of triangles using \cref{eq:Ricci_Forman}, \cref{eq:Forman_scal} where the true adjacency is replaced by the reconstructed one. Explicitly:
\[
(3\cdot 2)\gamma\sharp_{\Delta}(i) = d_{\rho}(i)\mathrm{R}(y_{i}) - \sum_{j}(A_{\rho})_{ij}(4 - d_{\rho}(i) - d_{\rho}(j)), 
\]
\noindent where the extra 2 factor on the LHS derives from counting each triangle twice in the formula $2\sharp_{\Delta}(i) = \sum_{j\sim i}\sharp_{\Delta}(i,j)$. For the results reported in \cref{table:triangle_append} we take $\gamma = 4$ in the weighting of triangles. We also note that the $0\%$ improvement over the CS-PhD dataset is to ascribe to the extremely low density of the true graph (with only 4 triangles overall): both methods reach low average distortion - albeit in this case not highly indicative. 


\textbf{Curvature correction}. Once the reconstructed adjacency $A_\rho$ (and hence a reconstructed graph $G_\rho$) is available, one can compute the node-wise Forman curvature, $\mathrm{F}_\rho (i)$ with $i\in V$. Since in our embeddings the curvature on the manifold $\mathrm{R}$ is a good proxy of the curvature of the graph, one can use the discrepancy $\vert \mathrm{R}(y_i) - \mathrm{F}_\rho(i)\vert$ to identify the points where the reconstruction is poor. Indeed, the discrepancy $\vert \mathrm{R}(y_i) - \mathrm{F}_\rho(i)\vert$ translate the quality of the reconstruction of the 2-hop neighborhood of the node $i$, by definition of Forman curvature. \emph{How to best exploit this additional information in reconstruction tasks and link prediction is of interest on its own and goes beyond the scope of the work}. However, we conducted preliminary experiments resorting on a simple curvature correction that acts as follows: 
\begin{itemize}
    \item Compute $err_i = \vert \mathrm{R}(y_i) - \mathrm{F}_\rho(i)\vert$ for each node $i$
    \item Identify the nodes where the error $err_i$ is bigger, e.g. the nodes where $err_i$ is above the $90\%$ percentile. 
    \item For these nodes, increase / decrease the distance threshold that governs the edge, obtaining a new graph $G'_\rho$ that is locally modified from the reconstructed graph $G_\rho$. 
    \item Compute the curvature of nodes of the new graph $G'_\rho$ and compare it with the curvature of the corresponding points on the manifolds. If the discrepancy between the curvature decreases, accept the change. Otherwise reject it. 
\end{itemize}

\noindent We have tested this method on WebEdu attaining a $22.6 \%$ improvement in the reconstruction (see \cref{fig:web_edu_recon}).

 \begin{table}
 \centering
\begin{tabular}{ |c||c|c|c|c|  }
 \hline
\multicolumn{5}{|c|}{\textbf{Triangle distortion}} \Tstrut\Bstrut\\
\hline
 \multicolumn{1}{|c||}{ } & \multicolumn{1}{|c|}{\textbf{Aves-Wildbird} } &  \multicolumn{1}{|c|}{\textbf{CS-PhD}} & \multicolumn{1}{|c|}{\textbf{WebEdu}} & 
 \multicolumn{1}{|c|}{\textbf{Facebook}} \Tstrut\Bstrut\\
  \multicolumn{1}{|c||}{ $|\Delta|$, avg $\Delta_i$ } & \multicolumn{1}{|c|}{ 9270, 70.76  } &  \multicolumn{1}{|c|}{ 4, 0.004 } & \multicolumn{1}{|c|}{ 10058, 3.31  } & 
 \multicolumn{1}{|c|}{ 1612010, 399.2 } \Tstrut\Bstrut\\
  \hline

$\mathbb{H}^5\times \mathbb{H}^5$ & 0.212 & 0.004 & 0.658 & 0.518\\[2pt]

$\mathbb{H}^5\times \mathbb{H}^5\!\heter\! $ & 0.108 & 0.004 & 0.369 & 0.414\\
  \hline

\end{tabular}
\caption{\label{table:triangle_append} Performance on the triangle counting task.}
\end{table}

\subsection{Manifold random graphs} Here we comment more on our random graph generation. We consider a three dimensional hyperbolic space $\mathbb{H}^{3}$ and we follow the sampling procedure adopted for example in \citet{cruceru2020computationally}: namely, one samples uniformly tangent vectors at some fixed point (this is not important due to homogeneity) and then use the exponential map to generate points inside the manifold. We observe that this approach is biased since it does not account for the underlying geometry (i.e. the Riemannian measure) but only sees the `flat' geometry of the tangent spaces. Nonetheless, for our purposes of random generation we prefer to stick to this easier uniform tangent sampling. In fact, if we actually encoded the hyperbolic Riemannian measure, then the sampling would have an even more manifest scale-free profile (since points on the Poincar\'{e} disk would be on average closer to the boundary) as shown in \citet{cruceru2020computationally}. 

Once we have a point cloud inside the Cartesian product of the Poincar\'{e} disk (and our spherically symmetric extension $\mathbb{H}^{3}\heter$), we construct the \emph{nearest neighbour graph} $G$ with adjacency $A$ using a \textbf{distance threshold} $\rho$, meaning that $A_{ij} = 1$ if points $y_{i}$ and $y_{j}$ are at geodesic distance smaller or equal than $\rho$. In general, graphs uniformly sampled from a hyperbolic geometry without accounting for heterogeneous curvature will exhibit small-world network features as power-law degree distribution \cite{krioukov2010hyperbolic}, however they will generally lack community structure (cliques). We then set the following:

\textbf{Goal:} \emph{Sample random graphs of 500 nodes that have one large community (as measured by the existence of a clique of $\sim 45-50$  nodes) while preserving the scale-free behaviour of the density (degree).} 

For the statistics reported below we sample 100 random graphs for each given threshold. 

\paragraph{Approach one: Increase the distance threshold $\rho$} In one case, where we simply sample points from the hyperbolic space, we consider increasing thresholds $\rho$ to improve the average density. While this allows for formation of dense community structures and achieves higher mean clustering, the higher density is distributed uniformly on the graph due to the homogeneity of the continuous manifold. In fact, the variance of the degree increases dramatically too, ultimately resulting in a failure to preserve the scale-free property when arriving to large cliques. On the other hand, the variance of the clustering coefficient decreases by more than 50 $\%$, highlighting how now in most of the graph the probability of triangle formation is uniformly high. This is all summed up in the statistics reported in \cref{table_generated_graph_stats_dist}.

\begin{table}[h!]
\begin{tabular}{|l|lllll|}
\hline
Homogeneous           & $\rho =1$            & $\rho =1.2$          & $\rho =1.5$           & $\rho =1.7$           & $\rho =2$         \\ \hline
variance degree     & 6.79 $\pm$ 1.85 & 10.72 $\pm$ 3.05 & 18.27 $\pm$ 3.99 & 23.71 $\pm$ 5.76 & 40.36 $\pm$ 8.95  \\ \hline
mean degree         & 7.33 $\pm$ 1.47 & 11.34 $\pm$ 2.73 & 19.33 $\pm$ 3.70 & 25.94 $\pm$ 5.64 & 47.24 $\pm$ 10.23 \\ \hline
variance clustering & 0.29 $\pm$ 0.01 & 0.27 $\pm$ 0.01  & 0.22 $\pm$ 0.01  & 0.19 $\pm$ 0.01  & 0.15 $\pm$ 0.01   \\ \hline
mean clustering     & 0.42 $\pm$ 0.02 & 0.53 $\pm$ 0.02  & 0.63 $\pm$ 0.01  & 0.67 $\pm$ 0.01  & 0.73 $\pm$ 0.01   \\ \hline
size largest clique & 10.6 $\pm$ 1.97 & 15.28 $\pm$ 3.18 & 23.84 $\pm$ 4.58 & 29.22$ \pm$ 6.07 & 49.96 $\pm$ 9.39  \\ \hline
\end{tabular}
\caption{\label{table_generated_graph_stats_dist} Statistics of the random graphs sampled from the homogeneous model $\mathbb{H}^{3}$ for different \emph{distance thresholds}. The formation of dense community structure can only occur uniformly at the cost of the scale-free property of the networks.}
\end{table}



\paragraph{Approach two: increase the connectivity among positively curved points} For the point cloud sampled uniformly\footnote{We point out here that the sampling occurs in a compact region. In the case of $\mathcal{R}$ we consider radial coordinates sampled in the interval $(0,2)$.} in $\mathbb{H}^{3}\heter$ we can also leverage the curvature information, meaning that now differently from the homogeneous space to each point $y_{i}$ we can also associate position-dependent curvature information $\mathrm{R}_{\alpha}(y_{i})$. In particular we fix $\alpha$ in \cref{eq:explicit_varphi} and some curvature threshold $\ell$ and assign a connection between any pair of points sampled from $\mathbb{H}^{3}\heter$ with both curvatures larger than $\ell$ and within a distance threshold $\rho$ we now vary again as above. We report the results in \cref{table_generated_graph_stats_curv}: we point out that now we reach a large dense clique (community) structure while still preserving the scale-free profile as shown in the degree distribution in \cref{fig:randomgraph_deg_curv}, the mean degree and its variance. The degree distribution is representative of the 100 runs, as it is obtained computing the average (or barycenter) of the degree distributions of all runs using Wasserstein distance. Wasserstein distance is sensitive to the shape and geometry of the probability distributions and therefore particularly suitable to compute the average of histograms, preserving their shape (\cite{agueh2011barycenters, cuturi2014fast, luise2019sinkhorn}).   Moreover, we observe how the variance of the clustering coefficient has not been affected significantly meaning that our sampling has managed to give rise to a highly heterogeneous density distribution. This is just a simple application of how heterogeneous manifolds could potentially be used to generate believable graphs that share many properties with real ones.

\begin{table}[H]
\begin{tabular}{|l|lllll|}
\hline
Heterogeneous           & $\rho =  1$           & $\rho =2.5$          & $\rho = 4$            & $\rho = 5.5 $          & $\rho = 7.0 $          \\ \hline
variance degree     & 6.34 $\pm$ 1.7  & 8.26 $\pm$ 2.8   & 12.9 $\pm$ 3.2   & 16.21 $\pm$ 2.7  & 19.4 $\pm$ 3.51  \\ \hline
mean degree         & 7.09 $\pm$ 1.37 & 8.46 $\pm$ 2.20  & 10.99 $\pm$ 2.58 & 12.33 $\pm$ 1.81 & 14.06 $\pm$ 2.47 \\ \hline
variance clustering & 0.28 $\pm$ 0.01 & 0.28 $\pm$ 0.01  & 0.28 $\pm$ 0.01  & 0.29 $\pm$ 0.01  & 0.30 $\pm$ 0.01  \\ \hline
mean clustering     & 0.42 $\pm$ 0.02 & 0.43 $\pm$ 0.02  & 0.45 $\pm$ 0.02  & 0.46 $\pm$ 0.02  & 0.47 $\pm$ 0.02  \\ \hline
size largest clique & 10.6 $\pm$ 2.29 & 11.94 $\pm$ 3.17 & 24.9 $\pm$ 6.31  & 35.9 $ \pm$ 5.72 & 47.2 $\pm$ 8.24  \\ \hline
\end{tabular}
\caption{\label{table_generated_graph_stats_curv} Statistics of the random graphs sampled from our heterogeneous model $\mathbb{H}^{3}\heter$ for different \emph{curvature thresholds}. We note how compared to the case where sampled points all come with same curvature from the ambient space, in this case we can attain large clique sizes without affecting significantly the variance of the clustering and of the degree distribution.}
\end{table}

\end{document}